\documentclass{article}
\usepackage{amsmath,amsfonts,bm}

\def\eqref#1{equation~\ref{#1}}
\def\1{\bm{1}}

\DeclareMathAlphabet{\mathsfit}{\encodingdefault}{\sfdefault}{m}{sl}
\SetMathAlphabet{\mathsfit}{bold}{\encodingdefault}{\sfdefault}{bx}{n}

\DeclareMathOperator*{\argmax}{arg\,max}

\usepackage{leaplab}
\usepackage{placeins}
\usepackage[table,x11names]{xcolor}
\usepackage[utf8]{inputenc}
\usepackage[T1]{fontenc}
\usepackage{url}
\usepackage{booktabs}
\usepackage{amsfonts}
\usepackage{wrapfig}
\usepackage{subcaption}
\usepackage{amsmath}
\usepackage{float} 
\usepackage{amssymb}
\usepackage{amsthm}
\usepackage{bm}
\usepackage{nicefrac}
\usepackage{wrapfig}
\usepackage{microtype}
\usepackage{graphicx}
\usepackage{caption}
\usepackage{subcaption}
\usepackage{algorithm}
\usepackage{algorithmic}
\usepackage{multirow}
\usepackage[most]{tcolorbox}
\usepackage{enumerate}
\usepackage{pifont}
\usepackage{amsthm}
\newtheorem{assumption}{Assumption}
\usepackage{fontawesome5}
\usepackage{enumitem}
\usepackage{marvosym}
\usepackage[colorlinks,linkcolor=red,citecolor=blue,urlcolor=blue]{hyperref}
\usepackage{graphicx} 
\usepackage{hyperref}
\usepackage{url}
\usepackage{multirow}
\usepackage{soul}
\usepackage{multicol}
\newtheorem{theorem}{Theorem}[section]
\newtheorem{proposition}[theorem]{Proposition}

\theoremstyle{remark}

\definecolor{lightgrey}{RGB}{247, 247, 247}
\newenvironment{leapabstract}{
  \begin{tcolorbox}[
    colback=lightgrey,
    colframe=white,
    boxrule=0pt,
    arc=10pt,
    left=16pt,
    right=16pt,
    top=12pt,
    bottom=12pt,
    width=\textwidth,
    enlarge left by=0mm,
    before skip=10pt,
    after skip=10pt
  ]
}{
  \end{tcolorbox}
}

\makeatletter
\def\icmldate#1{\gdef\@icmldate{#1}}
\icmldate{\today}
\makeatother

\makeatletter
\fancypagestyle{fancytitlepage}{
  \fancyhead{}
  \lhead{\includegraphics[height=1.5cm]{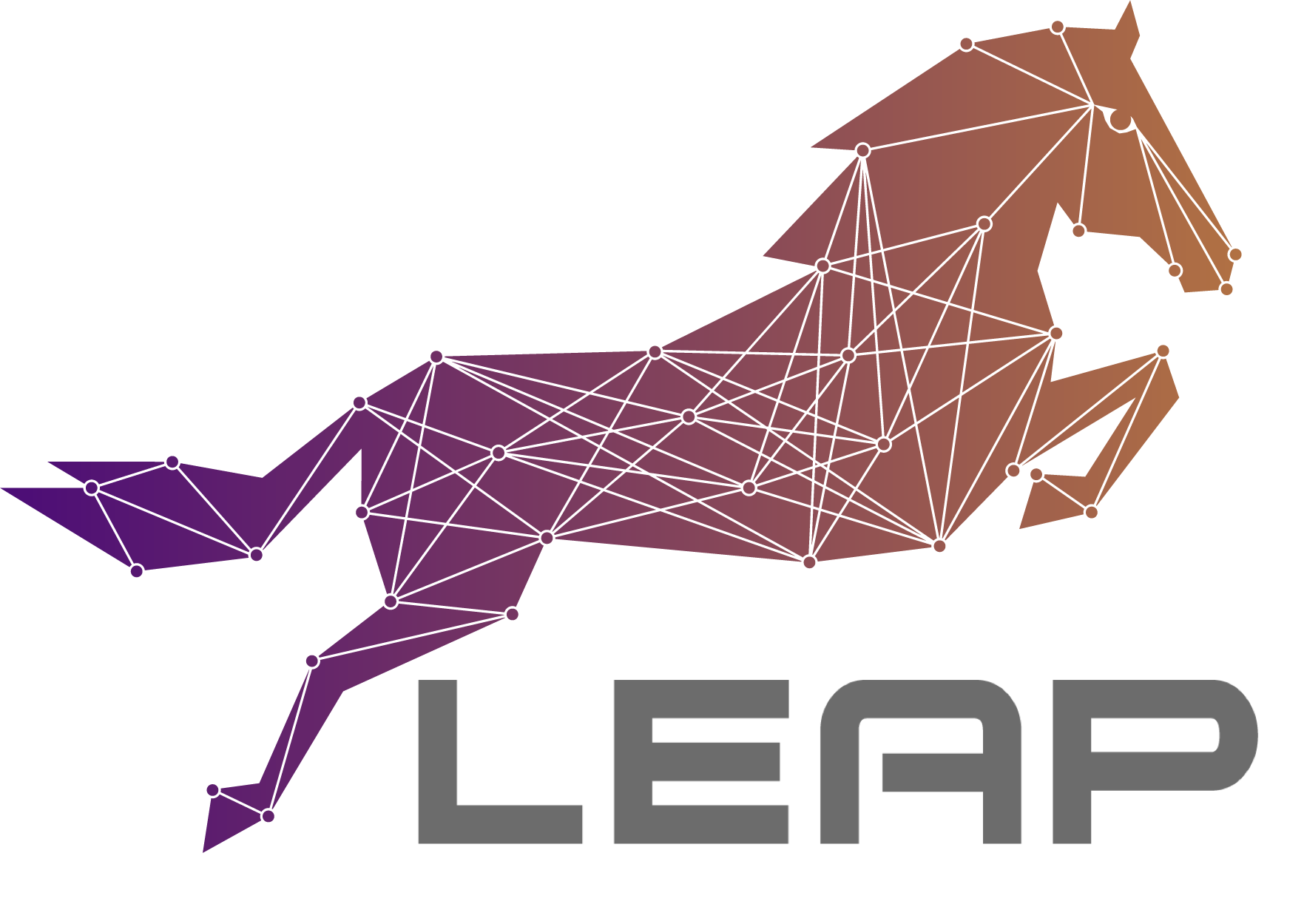}}
  \rhead{\it \@icmldate}
  \cfoot{}
}
\makeatother

\icmltitlerunning{On-policy Distillation with Verifiable Reward}

\begin{document}

\icmltitle{On-policy Distillation with Verifiable Reward}

\begin{icmlauthorlist}
\mbox{Wenze Lin$^{\,1,2\,*\,\dagger}$},
\mbox{Jiale Zhao$^{\,1,3,*}$},
\mbox{Xitai Jiang$^{\,1,2,*}$},
\mbox{Songde Rao$^{4}$},
\mbox{Yining Li$^{\,1,2}$},
\mbox{Shenzhi Wang$^{1}$},
\mbox{Bingxiang He$^{5}$},
and \mbox{Gao Huang$^{1,\textrm{\Letter}}$}
\end{icmlauthorlist}

$^{1\,}$LeapLab, Tsinghua University \quad
$^{2\,}$Qiuzhen College, Tsinghua University \quad
$^{3\,}$Beihang University \quad
$^{4\,}$SMS, Peking University \quad
$^{5\,}$NLPLab, Tsinghua University

$^{*}$ Equal Contribution \quad $^{\dagger}$ Project Lead \quad $^{\textrm{\Letter}}$ Corresponding Author

\icmlcorrespondingauthor{linwz25@mails.tsinghua.edu.cn, gaohuang@tsinghua.edu.cn}{}

\vskip .3in

\printNotice{}

\begin{leapabstract}
Reinforcement Learning with Verifiable Rewards (RLVR) and on-policy distillation (OPD) have become two widely adopted paradigms for post-training large language models. However, RLVR suffers from sparse task-level feedback, while OPD provides dense token-level guidance but ignores trajectory correctness, limiting its performance to that of the teacher. Combining them is a promising direction: OPD supplies dense supervisory signals, while RLVR provides task-level correctness. Nevertheless, existing integrations often rely on weighted combination or heuristic switching, introducing extra hyperparameters and trade-offs. We propose \textbf{On-policy Distillation with Verifiable Reward (OPDVR)}, a simple yet effective method that seamlessly combines OPD and RLVR without adding any hyperparameters. We first reformulate the implicit reward of sampled-token OPD based on trajectory correctness, then apply a ReLU gating mechanism to ensure that correct trajectories receive non-negative rewards and incorrect ones receive non-positive rewards—thereby aligning the distillation signal with task success while preserving the teacher's distributional guidance. Furthermore, our modification transforms sampled-token OPD into a proper RLVR method, making it readily combinable with any policy gradient algorithm, such as GRPO. Experiments on six reasoning benchmarks show that OPDVR consistently outperforms standard OPD. Our code is available at \url{https://github.com/LeapLabTHU/OPDVR}.
\end{leapabstract}

\section{Introduction}\label{sec:introduction}
Reinforcement Learning with Verifiable Rewards (RLVR) has become a prevalent and effective post-training paradigm for reasoning tasks \citep{DeepSeekR1,DeepSeekMath,jaech2024openai,trinh2024solving,yang2024qwen2}. RLVR delivers explicit reward signals based on task outcomes, such as correct mathematical final answers or compilable code, which directly align model optimization with core task objectives. However, these rewards are typically sparse, providing little supervision for intermediate steps and making credit assignment challenging.
Recently, On-policy Distillation (OPD) has emerged as a promising post-training paradigm \citep{xu2026deepseek,xiao2026mimo,yang2025qwen3,zeng2026glm}. Unlike RLVR, which relies on sparse outcome-level rewards, OPD leverages a teacher model to provide dense token-level guidance. However, OPD's objective is purely distributional---it drives the student to mimic the teacher's output distribution without considering whether the generated response is correct or incorrect, limiting the student's performance to that of the teacher.

The complementary limitations of RLVR and OPD make combining them a natural and promising direction, where OPD provides fine-grained dense guidance and RLVR offers reliable task-level correctness supervision. Existing approaches typically treat OPD and RLVR as two separate methods, either combining them explicitly or selectively applying one based on heuristic criteria. Despite their empirical effectiveness, such designs often introduce extra hyperparameters and heuristic trade-offs \citep{cai2026h,yang2026self,li2026unifying,wang2026distilled}.

In this work, we propose \textbf{On-policy Distillation with Verifiable Reward (OPDVR)}. We apply an extremely simple ReLU gating mechanism to sampled-token OPD, seamlessly combining OPD and RLVR without introducing any hyperparameters. We first revisit sampled-token OPD from an RLVR perspective and provide a mathematical reformulation of sample-token OPD's reward. From the RLVR view, if we separate by trajectory correctness, sampled-token OPD can be viewed as applying token-level supervision by weighting binary task correctness rewards with teacher-student distribution discrepancies. Specifically, for sampled tokens that constitute correct reasoning trajectories, the model is assigned a token-level reward $+1$ weighted by $\log(\pi_T / \pi_\theta)$; for tokens leading to incorrect trajectories, the model receives a reward $-1$ weighted by $\log(\pi_\theta / \pi_T)$. However, we identify a critical limitation of this inherent reward design: both $\log(\pi_T / \pi_\theta)$ and $\log(\pi_\theta / \pi_T)$ are unbounded and can be either positive or negative. However, a key empirical principle shared by mainstream RLVR algorithms ~\citep{schulman2017proximal,DeepSeekR1,DeepSeekMath} is that all tokens leading to a correct final outcome should receive non-negative advantages, while tokens leading to an incorrect outcome should receive non-positive advantages. This principle ensures that every token in a correct trajectory is treated as a valid prediction and reinforced accordingly, while every token in an incorrect trajectory is treated as an erroneous prediction and suppressed accordingly. Sampled-token OPD violates this principle: on correct trajectories, a negative value of \(\log(\pi_T / \pi_\theta)\) penalizes valid token behaviors, while on incorrect trajectories, a negative value of \(\log(\pi_\theta / \pi_T)\) encourages erroneous token predictions.

Motivated by this gap, we use a ReLU gating mechanism to standardize the reward signs according to trajectory correctness: it enforces non-negative $\log(\pi_T / \pi_\theta)$ values for all verified correct trajectories and ensures non-negative $\log(\pi_\theta / \pi_T)$ values for all incorrect trajectories. This minimal correction directly converts conventional sampled-token OPD into a valid RLVR method while preserving the guidance from the teacher model. Specifically, for correct trajectories, larger $\log(\pi_T / \pi_\theta)$ values—indicating accurate student predictions where the teacher is more confident than the student—yield stronger rewards. For incorrect trajectories, larger $\log(\pi_\theta / \pi_T)$ values—corresponding to wrong student predictions where the student is more confident than the teacher—incur heavier penalties. This learning paradigm aligns well with human learning intuition, which\textbf{ reinforces reliable correct behaviors and suppresses overconfident erroneous predictions}. Furthermore, by reformulating sampled-token OPD into an RLVR formulation, our OPDVR supports seamless integration with existing prominent RL algorithms, such as GRPO, DAPO and PPO. To instantiate this, we combine OPDVR with GRPO to obtain a new variant, which we call \textbf{Group Relative Policy Distillation (GRPD)}. Overall, our method endows the model with dual supervisory signals: the rigorous task-level verifiable correctness from RLVR and the dense token-level distributional guidance from teacher distillation, leading to more robust and effective post-training optimization for reasoning tasks. Our experiments show that both OPDVR and GRPD consistently outperform OPD across all benchmarks.

% \begin{figure}[t]  % t=top, b=bottom, h=here, p=单独一页
%     \centering
%     \includegraphics[width=1.0\textwidth]{figs/compare_final.pdf}
%     \caption{Comparison of reward formulations between OPD and OPDVR.}
%     \label{fig:pipeline}
% \end{figure}

\begin{figure}[htbp]
    \centering
    \begin{minipage}[t]{0.49\textwidth}
        \centering
        \includegraphics[width=\linewidth]{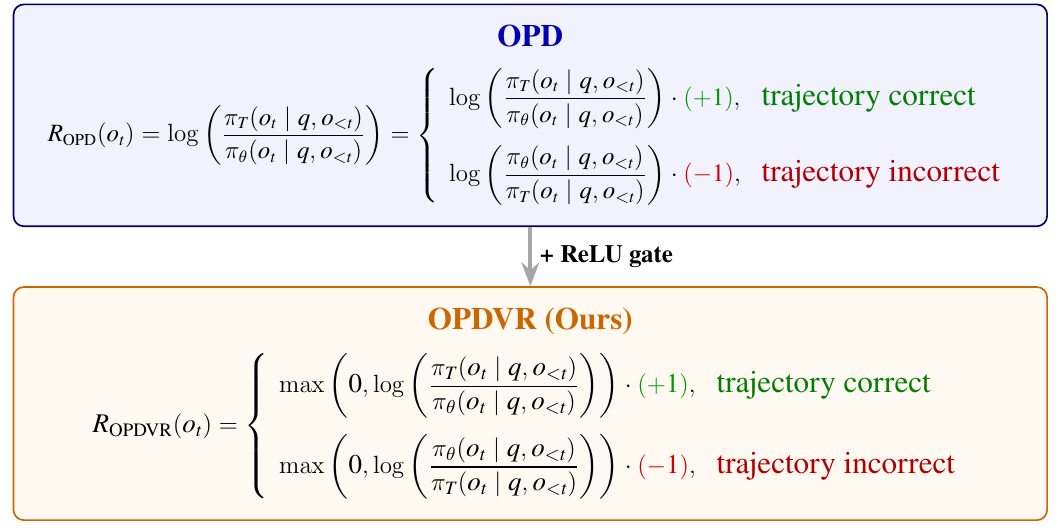}
    \end{minipage}
    \hfill
    \begin{minipage}[t]{0.49\textwidth}
        \centering
        \includegraphics[width=\linewidth]{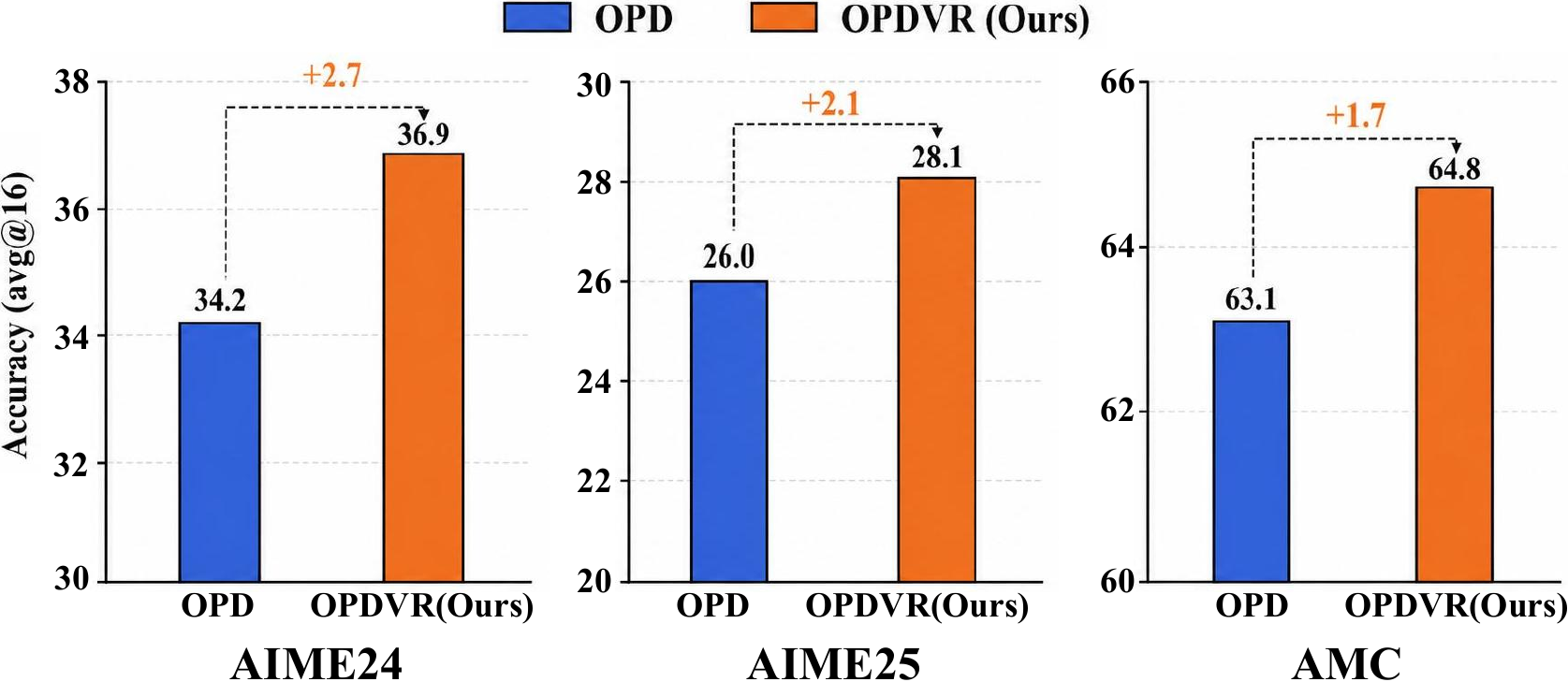}
    \end{minipage}
    \caption{Left: Overview of OPDVR. The ReLU gating mechanism ensures correct trajectories receive non-negative rewards and incorrect ones receive non-positive rewards, while preserving the teacher's distributional guidance. Right: Results on AIME24, AIME25, and AMC under the same-architecture setting (Qwen3-4B $\leftarrow$ Qwen3-4B-RL), reported as avg@16 accuracy.}
    \label{fig:combined}
\end{figure}

\section{Related Work} \label{sec:Related Work}

\paragraph{Reinforcement Learning with Verifiable Rewards (RLVR)} Reinforcement Learning with Verifiable Rewards (RLVR) has emerged as a highly effective paradigm for post-training large language models~\citep{schulman2017proximal,jaech2024openai,trinh2024solving,yang2024qwen2,DeepSeekR1,DeepSeekMath}. RLVR leverages rule-based verifiers to provide objective, deterministic signals such as answer correctness or code compilation. Various extensions have been proposed to improve RLVR, such as controlling length~\citep{LASER,sui2025stop,xiang2025just,hammoud2025train,hou2025thinkprune,huang2026hapo} and stabilizing entropy~\citep{petrenko2026entropy,EntroPIC,jiang2025rethinking,Su2025CEGPPOCE}. However, a fundamental limitation persists: the reward is sparse, leading to a severe credit assignment problem.

\paragraph{On-policy Distillation} On-policy Distillation (OPD) has recently emerged as an efficient post-training paradigm that provides dense, token-level supervisory signals by distilling a teacher model's output distribution into the student~\citep{xu2026deepseek,xiao2026mimo,yang2025qwen3,zeng2026glm}. Unlike RLVR's sparse outcome-based rewards, OPD offers fine-grained guidance at every generation step. A growing line of work has explored various aspects of OPD, including its failure modes~\citep{fu2026revisiting,li2026rethinking}, how to determine which tokens are worth training~\citep{xing2026trust,jin2026entropy}, and its training stability issues in practice~\citep{oh2026kl}. However, OPD's objective is purely distributional: it drives the student to mimic the teacher's output distribution without considering whether the generated response is correct or incorrect. This observation naturally raises the question of whether OPD can be combined with RLVR, where trajectory correctness provides the ultimate optimization signal. Combining the dense supervision of OPD with the task-level verifiability of RLVR thus represents a promising and increasingly relevant direction.

\paragraph{Combining OPD with RLVR} Given the complementary strengths of OPD (dense token-level guidance) and RLVR (task-level correctness), several recent works have attempted to combine them. These include directly weighting the two objectives~\citep{Hubotter2026ReinforcementLV}, applying OPD and RLVR separately depending on the sign of the advantage~\citep{wang2026distilled,cai2026h}, aligning outcome rewards with process-level distillation signals~\citep{hou2026uni},  and leveraging teacher-student probability ratios for to weight the advantage in GRPO~\citep{yang2026self}. Despite their empirical gains, these approaches often treat OPD and RLVR as distinct objectives or losses that must be either explicitly combined or selectively applied based on heuristic criteria, often introducing additional hyperparameters for balancing the two terms, or relies on heuristic trade-offs. Instead, our method implicitly transforms sampled-token OPD into a proper RLVR method using a simple gated mechanism, while preserving the teacher's distributional guidance—achieving the benefits of both paradigms without explicit multi-objective balancing.

\paragraph{Identifying the Misalignment between Teacher and Verifier}
Several recent works have identified the inconsistency between teacher distillation signals and verifier feedback. Uni-OPD~\citep{hou2026uni} discards prompts whose teacher sequence-level reward gap between positive and negative trajectories is smaller than a safety margin. RG-OPD~\citep{akhondzadeh2026reward} directly discards entire trajectories where the teacher signal disagrees with the verifier. However, they both operate at the trajectory level, discarding potentially valuable tokens in inconsistent trajectories. SG-OPD~\citep{xu2026sg} moves to the token level by extrapolating the distillation signal on agreeing tokens and interpolating on conflicting ones. Yet, it only applies a scalar coefficient to scale the conflicting signals, without actually eliminating their undesirable effects. In contrast, we show that a simple ReLU gate on the sampled token suffices to remove conflicting signals, directly transforming sampled-token OPD into a proper RLVR method that reinforces reliable correct predictions and suppresses overconfident errors, without introducing any extra hyperparameters or heuristics.

\section{Preliminaries} \label{sec:Preliminaries}

\subsection{Reinforcement Learning with Verifiable Rewards (RLVR)}

Reinforcement Learning with Verifiable Rewards (RLVR) optimizes a policy using reward signals that can be automatically verified against ground truth. In mathematical reasoning tasks, the reward indicates whether the final answer is correct. A common choice for REINFORCE is \( R = +1 \) for a correct trajectory and \( R = -1 \) for an incorrect one, while many other RLVR methods (including GRPO) use \( R = +1 \) for correct and \( R = 0 \) for incorrect.

Consider a policy \(\pi_\theta\) that generates a response \(o = \{o_1, \dots, o_{|o|}\}\) given a prompt \(q\). The RLVR objective is to maximize the expected reward:
\[
\mathcal{J}(\theta) = \mathbb{E}_{o \sim \pi_\theta(\cdot \mid q)} \left[ R \right],
\]
where \(R\) is the trajectory-level reward.

To optimize this objective via gradient ascent, we compute the gradient:
\[
\nabla_\theta \mathcal{J}(\theta) = \nabla_\theta \mathbb{E}_{o \sim \pi_\theta(\cdot \mid q)} \left[ R \right].
\]

Using the REINFORCE (log-derivative) trick, we rewrite this as:
\[
\nabla_\theta \mathcal{J}(\theta) = \mathbb{E}_{o \sim \pi_\theta(\cdot \mid q)} \left[ R \cdot \nabla_\theta \log \pi_\theta(o \mid q) \right],
\]
where \(\log \pi_\theta(o \mid q) = \sum_{t=1}^{|o|} \log \pi_\theta(o_t \mid q, o_{<t})\).

We sample a single response \(o \sim \pi_\theta(\cdot \mid q)\) to obtain a Monte Carlo estimate:
\[
\nabla_\theta \mathcal{J}(\theta) \approx R \cdot \sum_{t=1}^{|o|} \nabla_\theta \log \pi_\theta(o_t \mid q, o_{<t}).
\]

This gradient can be implemented by minimizing the following loss over the sequence:
\[
\mathcal{L}_{\text{RLVR}}(\theta) = - R \cdot \sum_{t=1}^{|o|} \log \pi_\theta(o_t \mid q, o_{<t}).
\]

When \(R = +1\), minimizing this loss increases the probability of the entire response (reward). When \(R = -1\), it decreases the probability (penalty). This formulation embodies the core RL principle: correct actions are rewarded, incorrect ones are penalized.

The REINFORCE formulation above suffers from high variance, as it relies on a single trajectory-level reward. A more stable alternative is Group Relative Policy Optimization (GRPO) \citep{DeepSeekR1,DeepSeekMath}. Specifically, for a group of \(G\) responses \(\{o_1, \dots, o_G\}\) sampled from the policy, GRPO normalizes each response's reward against the group mean and standard deviation:
\[
\hat{A}_{i,t} = \frac{R_i - \text{mean}(\{R_j\}_{j=1}^G)}{\text{std}(\{R_j\}_{j=1}^G)}.
\]
Note that GRPO typically uses \( R \in \{0, 1\} \) (1 for correct, 0 for incorrect) rather than \( \{-1, +1\} \), since the group-relative normalization naturally centers the advantages around zero.

This group-relative advantage reduces variance and serves as a stable, computationally efficient alternative to direct binary rewards. For simplicity, we present the on-policy version of GRPO without clipping or KL penalty; the loss is:
\[
\mathcal{L}_{\text{GRPO}}(\theta) = - \frac{1}{G} \sum_{i=1}^G \frac{1}{|o_i|} \sum_{t=1}^{|o_i|} \hat{A}_{i,t} \log \pi_\theta(o_{i,t} \mid q, o_{i,<t}).
\]

\subsection{On-policy Distillation}

On-policy Distillation (OPD) optimizes a student policy \(\pi_\theta\) toward a teacher policy \(\pi_T\) by minimizing their reverse KL divergence:
\[
\min_{\theta} \ \mathbb{E}_{o \sim \pi_\theta(\cdot \mid q)} \left[ \sum_{t=1}^{|o|} D_{\text{KL}}(\pi_\theta(\cdot \mid q, o_{<t}) \parallel \pi_T(\cdot \mid q, o_{<t})) \right].
\]

The \textbf{full-vocabulary OPD} computes this exact KL at every token position:
\[
\mathcal{L}_{\text{OPD}}^{\text{full}}(\theta) = \mathbb{E}_{o \sim \pi_\theta(\cdot \mid q)} \left[ \sum_{t=1}^{|o|} \sum_{v \in \mathcal{V}} \pi_\theta(v \mid q, o_{<t}) \log \frac{\pi_\theta(v \mid q, o_{<t})}{\pi_T(v \mid q, o_{<t})} \right].
\]
This loss is directly minimized. In practice, the log-ratio \(\log(\pi_\theta / \pi_T)\) is treated as a constant reward with stop-gradient when computing gradients, which is equivalent to the exact gradient of the reverse KL via the log-derivative trick. However, computing the reverse KL over the full vocabulary at every token position is computationally expensive, especially for large language models with very large vocabularies.

A practical compromise is \textbf{Top-\(k\) OPD}, which restricts the KL computation to the \(k\) tokens with the highest student probabilities:
\[
\mathcal{L}_{\text{OPD}}^{\text{top-}k}(\theta) = \mathbb{E}_{o \sim \pi_\theta(\cdot \mid q)} \left[ \sum_{t=1}^{|o|} D_{\text{KL}}\left(\bar{\pi}_\theta^{(S_t)} \parallel \bar{\pi}_T^{(S_t)}\right) \right],
\]
where \(S_t\) is the set of top-\(k\) tokens.

The most widely adopted variant is \textbf{sampled-token OPD}. It uses a single Monte Carlo sample \(o_t \sim \pi_\theta(\cdot \mid q, o_{<t})\) to approximate the reverse KL gradient:
\[
\mathcal{L}_{\text{OPD}}^{\text{sample}}(\theta) = \mathbb{E}_{o \sim \pi_\theta(\cdot \mid q)} \left[ \sum_{t=1}^{|o|} \log \frac{\pi_\theta(o_t \mid q, o_{<t})}{\pi_T(o_t \mid q, o_{<t})} \right].
\]
The log-ratio \(\log(\pi_\theta / \pi_T)\) is treated as a constant reward with stop-gradient.

\section{Method} \label{sec:Method}
\subsection{Revisiting Sampled-Token OPD from an RLVR Perspective}

We begin with the sampled-token OPD loss. Recall that sampled-token OPD approximates the reverse KL gradient using a single Monte Carlo sample. The loss over the entire response \(o = \{o_1, \dots, o_{|o|}\}\) is defined as:
\[
\mathcal{L}_{\text{OPD}}^{\text{sample}}(\theta) = \mathbb{E}_{o \sim \pi_\theta(\cdot \mid q)} \left[ \sum_{t=1}^{|o|} \log \frac{\pi_\theta(o_t \mid q, o_{<t})}{\pi_T(o_t \mid q, o_{<t})} \right].
\]

For a sampled response \(o \sim \pi_\theta(\cdot \mid q)\), the per-token loss is:
\[
\ell_t = \log \frac{\pi_\theta(o_t \mid q, o_{<t})}{\pi_T(o_t \mid q, o_{<t})}.
\]

The gradient of the sequence-level loss with respect to \(\theta\) is:
\[
\nabla_\theta \mathcal{L}_{\text{OPD}}^{\text{sample}}(\theta) = \sum_{t=1}^{|o|} \log \frac{\pi_\theta(o_t \mid q, o_{<t})}{\pi_T(o_t \mid q, o_{<t})} \cdot \nabla_\theta \log \pi_\theta(o_t \mid q, o_{<t}).
\]

We now compare this with the standard RLVR gradient. Recall from Section \ref{sec:Preliminaries} that the RLVR loss for a trajectory with reward \(R\) is:
\[
\mathcal{L}_{\text{RLVR}}(\theta) = - R \cdot \sum_{t=1}^{|o|} \log \pi_\theta(o_t \mid q, o_{<t}).
\]

Its gradient is:
\[
\nabla_\theta \mathcal{L}_{\text{RLVR}}(\theta) = - R \cdot \sum_{t=1}^{|o|} \nabla_\theta \log \pi_\theta(o_t \mid q, o_{<t}).
\]

We observe that the sampled-token OPD gradient shares the same form as the RLVR gradient, where the log-ratio term corresponds to the reward coefficient. By matching the coefficients of 
$\nabla_\theta \log \pi_\theta(o_t|q,o_{<t})$, we obtain:
\[
- R = \log \frac{\pi_\theta(o_t \mid q, o_{<t})}{\pi_T(o_t \mid q, o_{<t})}.
\]

this observation allows us to reinterpret the log-ratio term as an implicit token-level reward $R_{\text{OPD}}(o_t)$
\[
R_{\text{OPD}}(o_t) = - \log \frac{\pi_\theta(o_t \mid q, o_{<t})}{\pi_T(o_t \mid q, o_{<t})} = \log \frac{\pi_T(o_t \mid q, o_{<t})}{\pi_\theta(o_t \mid q, o_{<t})}.
\]

However, unlike RLVR, where the reward sign is determined by the verifier outcome, the sign of Ropd is solely determined by the teacher-student probability ratio. To analyze their alignment, we consider two cases based on trajectory correctness:
\[
R_{\text{OPD}}(o_t) = 
\begin{cases}
\log \dfrac{\pi_T(o_t \mid q, o_{<t})}{\pi_\theta(o_t \mid q, o_{<t})} \cdot (+1), & \text{trajectory correct}, \\[10pt]
\log \dfrac{\pi_\theta(o_t \mid q, o_{<t})}{\pi_T(o_t \mid q, o_{<t})} \cdot (-1), & \text{trajectory incorrect}.
\end{cases}
\]

In other words, on correct trajectories, the gradient is weighted by \(\log(\pi_T / \pi_\theta)\); on incorrect trajectories, it is weighted by \(\log(\pi_\theta / \pi_T)\) with a negative sign.

The empirical RLVR principle requires that correct trajectories receive positive rewards and incorrect trajectories receive negative rewards. However, the sign of \(\log(\pi_T / \pi_\theta)\) is determined by whether \(\pi_T > \pi_\theta\) or \(\pi_T < \pi_\theta\), which is independent of trajectory correctness. This creates two failure cases: correct trajectories can receive negative rewards when \(\pi_T < \pi_\theta\), and incorrect trajectories can receive positive rewards when \(\pi_T > \pi_\theta\). Both violate the RL principle.
\subsection{OPDVR: A Simple Gated Mechanism}
\begin{figure}[t]  % t=top, b=bottom, h=here, p=单独一页
    \centering
    \includegraphics[width=1.0\textwidth]{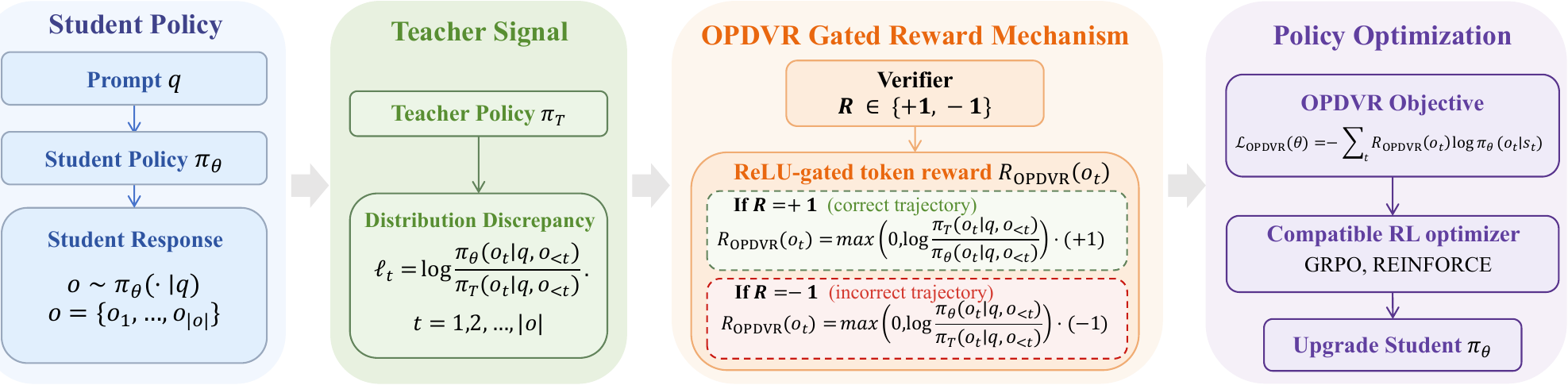}
    \caption{Training pipeline of OPDVR.}
    \label{fig:pipeline}
\end{figure}

We propose On-policy Distillation with Verifiable Reward (OPDVR) to resolve this issue. The solution is \textbf{extremely simple}: apply a ReLU gate to enforce RLVR compliance on the sampled token while preserving the teacher's distributional guidance.

The reward can be written as: 
\[
R_{\text{OPDVR}}(o_t) = 
\begin{cases}
\max\left(0, \log \dfrac{\pi_T(o_t \mid q, o_{<t})}{\pi_\theta(o_t \mid q, o_{<t})}\right) \cdot (+1), & \text{trajectory correct}, \\[12pt]
\max\left(0, \log \dfrac{\pi_\theta(o_t \mid q, o_{<t})}{\pi_T(o_t \mid q, o_{<t})}\right) \cdot (-1), & \text{trajectory incorrect}.
\end{cases}
\]

The corresponding loss is:
\[
\mathcal{L}_{\text{OPDVR}}(\theta) = - \sum_{t=1}^{|o|} R_{\text{OPDVR}}(o_t) \cdot \log \pi_\theta(o_t \mid q, o_{<t}).
\]

This gated mechanism ensures that:
\begin{itemize}
    \item Correct trajectories receive non-negative rewards, and incorrect trajectories receive non-positive rewards, aligning the reward sign with trajectory correctness.
    \item For tokens on correct trajectories, the reward magnitude is larger when the teacher is more confident than the student, i.e., $\log(\pi_T / \pi_\theta)$ is larger. This encourages the model to reinforce choices that are both correct and reliable according to the teacher.
    \item For tokens on incorrect trajectories, the penalty magnitude is larger when the student is more confident than the teacher, i.e., $\log(\pi_\theta / \pi_T)$ is larger. This forces the model to suppress overconfident mistakes.
\end{itemize}

\subsection{Interpreting the ReLU Gating Mechanism as a Conditional Mask}

In this subsection, we interpret the ReLU gating mechanism in OPDVR as a \textit{conditional token mask}: it selectively zeroes out gradient updates on tokens whose learning direction conflicts with the verifier signal. Standard sampled-token OPD lacks this mask, so it inevitably updates all tokens based solely on the teacher-student confidence ratio, regardless of trajectory correctness.

To see this, recall the standard OPD gradient for a sampled token $a$:
\begin{equation}
g_{\mathrm{OPD}} =
\underbrace{
\left[\log \frac{\pi_T}{\pi_\theta}\right]_+
\nabla_\theta \log \pi_\theta(a)
}_{\text{Term A: push up (teacher more confident)}}
-
\underbrace{
\left[\log \frac{\pi_\theta}{\pi_T}\right]_+
\nabla_\theta \log \pi_\theta(a)
}_{\text{Term B: pull down (student more confident)}}.
\end{equation}

The OPDVR gradient instead applies the ReLU gate conditioned on the verifier reward $R\in\{+1,-1\}$:
\begin{equation}
g_{\mathrm{OPDVR}} =
\begin{cases}
\text{Term A}, & R=+1,\\
-\text{Term B}, & R=-1.
\end{cases}
\end{equation}

Thus, compared to standard OPD, OPDVR \emph{removes} exactly two types of tokens whose updates conflict with the verifier:

\begin{itemize}
\item \textbf{Type I Conflicting Token (Correct Trajectory, $\pi_\theta > \pi_T$):}
When the trajectory is correct ($R=+1$), the student has already produced the right answer. If the student is more confident than the teacher on a specific token ($\pi_\theta > \pi_T$), then Term B is activated. Standard OPD applies a \textit{negative} gradient, pulling down this correct token's probability to match the teacher's lower confidence. This update works against the verifier signal in two ways:
\begin{enumerate}
    \item The student's higher confidence is \textit{consistent} with the correct outcome.
    \item Reducing the student's confidence on a correct answer introduces a distributional distortion that is not driven by task performance.
\end{enumerate}

\item \textbf{Type II Conflicting Token (Incorrect Trajectory, $\pi_T > \pi_\theta$):}
When the trajectory is incorrect ($R=-1$), the student has produced a wrong answer. If the teacher is more confident than the student on a specific token ($\pi_T > \pi_\theta$), then Term A is activated. Standard OPD applies a \textit{positive} gradient, pushing up this incorrect token's probability to align with the teacher. This update works against the verifier signal in two ways:
\begin{enumerate}
    \item The teacher's high confidence does not accompany a correct trajectory here.
    \item Increasing the probability of tokens on an incorrect trajectory reinforces a reasoning pattern that the verifier has marked as wrong.
\end{enumerate}
\end{itemize}

By masking out these two token types, OPDVR preserves the student's correct high-confidence predictions and withholds the teacher's guidance on wrong answers. The teacher still controls the \emph{magnitude} of the update via the log-ratio, but the verifier now determines the \emph{direction}---reinforce or suppress---so that the distillation process never works against the task reward.\textbf{ We provide a theoretical analysis of OPDVR in Appendix~\ref{sec:theory}.} Under a soft-value teacher assumption, we show that a single verifier-gated mirror update jointly improves task success and preserves teacher matching: task success does not decrease for any step size, and teacher KL does not increase for a modest step size. This provides a principled justification for why OPDVR can improve task performance without sacrificing distributional alignment.

\subsection{Group Relative Policy Distillation (GRPD)}\label{sec:GRPD}

Having established that OPDVR transforms sampled-token OPD into a proper RLVR method with a binary verifier reward, a natural extension is to replace this coarse binary signal with a more nuanced, group-relative advantage estimate like GRPO~\citep{DeepSeekR1,DeepSeekMath}.

Recall from Section~\ref{sec:Preliminaries} that GRPO computes a group-relative advantage \(\hat{A}_{i,t}\) for each token position \(t\) in response \(o_i\):
\[
\hat{A}_{i,t} = \frac{R_i - \text{mean}(\{R_j\}_{j=1}^G)}{\text{std}(\{R_j\}_{j=1}^G)},
\]
where \(R_i \in \{0, 1\}\) is the verifier reward for response \(o_i\), and \(G\) is the group size. The key property of \(\hat{A}_{i,t}\) is that it is positive for responses that are better than the group average, and negative for those that are worse—capturing relative performance within the sampled batch.

We now apply the same ReLU gating logic, but with the binary correctness sign \(R\) replaced by the group-relative advantage \(\hat{A}_{i,t}\). The GRPD reward becomes:

\[
R_{\text{GRPD}}(o_{i,t}) = 
\begin{cases}
\max\!\left(0, \log \dfrac{\pi_T(o_{i,t} \mid q, o_{i,<t})}{\pi_\theta(o_{i,t} \mid q, o_{i,<t})}\right) \cdot (+1), & \hat{A}_{i,t} > 0, \\[10pt]
\max\!\left(0, \log \dfrac{\pi_\theta(o_{i,t} \mid q, o_{i,<t})}{\pi_T(o_{i,t} \mid q, o_{i,<t})}\right) \cdot (-1), & \hat{A}_{i,t} < 0.
\end{cases}
\]

Equivalently, this can be written compactly as:
\[
R_{\text{GRPD}}(o_{i,t}) = \text{sign}(\hat{A}_{i,t}) \cdot \text{ReLU}\!\left( \text{sign}(\hat{A}_{i,t}) \cdot \log \frac{\pi_T(o_{i,t} \mid q, o_{i,<t})}{\pi_\theta(o_{i,t} \mid q, o_{i,<t})} \right),
\]
where \(\text{sign}(\hat{A}_{i,t})\) ensures the direction of the update is governed by the advantage sign.

The corresponding loss is:
\[
\mathcal{L}_{\text{GRPD}}(\theta) = - \frac{1}{G} \sum_{i=1}^G \frac{1}{|o_i|} \sum_{t=1}^{|o_i|} R_{\text{GRPD}}(o_{i,t}) \cdot \log \pi_\theta(o_{i,t} \mid q, o_{i,<t}).
\]

% 如果论文中尚未定义定理环境，请添加以下定义
% \usepackage{amsthm}
% \newtheorem{proposition}{Proposition}[section]
% \newtheorem{theorem}{Theorem}[section]
% \newtheorem{lemma}{Lemma}[section]
% \newtheorem{corollary}{Corollary}[section]
% \newtheorem{remark}{Remark}[section]

\section{Experiments}

\subsection{Main Experimental Settings}

We conduct experiments on both same-architecture and cross-architecture distillation settings to evaluate the effectiveness of our method.

\textbf{Same-architecture setting.} We use Qwen3-4B-nonthinking as the student model and distill it from a teacher model of the same architecture, which is obtained by training Qwen3-4B with GRPO on the DeepMath dataset~\citep{he2026deepmath}. Following \citet{yang2026learning}, we use the filtered subset of the DeepMath dataset consisting of 57k samples with difficulty level \(\ge 6\).

\textbf{Cross-architecture setting.} We use the DAPO-Math-17k dataset~\citep{yu2026dapo}. The teacher model is Qwen3-4B-base, fine-tuned with GRPO for 3 epochs on the same dataset. The student model is Qwen3-1.7B-base. The distillation training runs for 3 epochs.

\textbf{Baselines.} We compare against Sampled-Token OPD, Top-64 OPD, OPD+GRPO~\citep{xiao2026mimo}, RLSD~\citep{yang2026self}, ExOPD~\citep{yang2026learning}, and SG-OPD~\citep{xu2026sg}. Implementation details of all baselines are provided in Appendix~\ref{app:baselines}.

\textbf{Benchmarks.} We evaluate all models on six reasoning benchmarks: AIME24, AIME25, AMC, MATH500, Minerva, and OlympiadBench.

\subsection{Main Results}

\begin{table}[t]
\centering
\caption{Results on same-architecture distillation (Qwen3-4B $\leftarrow$ Qwen3-4B-RL). All models are evaluated with avg@16 accuracy.}
\label{tab:same_arch}
\resizebox{\textwidth}{!}{%
\begin{tabular}{l|cccccc|c}
\toprule
\textbf{Method} & \textbf{AIME24} & \textbf{AIME25} & \textbf{AMC} & \textbf{MATH500} & \textbf{Minerva} & \textbf{OlympiadBench} & \textbf{Avg.} \\
\midrule
Student (Qwen3-4B) & 24.0 & 15.8 & 60.8 & 80.9 & 27.6 & 42.9 & 42.0 \\
Teacher (Qwen3-4B-RL) & 36.0 & 29.0 & 65.9 & 87.0 & 35.4 & 49.3 & 50.4 \\
\midrule
Sampled-Token OPD & 34.2 & 26.0 & 63.1 & 85.5 & 31.6 & 46.5 & 47.8 \\
Top-64 OPD & 34.6 & 23.5 & 62.0 & 85.0 & 32.2 & 46.8 & 47.4 \\
OPD+GRPO & 33.1 & 24.2 & 64.4 & 85.2 & 32.9 & 46.5 & 47.7 \\
RLSD & 28.7 & 20.0 & 62.1 & 82.8 & 31.2 & 44.0 & 44.8 \\
ExOPD & 29.8 & 26.9 & 62.2 & \textbf{86.6}& \textbf{35.9} & 46.6 & 47.9 \\
SG-OPD & 31.6 & 26.2 & 61.2 & 84.7 & 35.2 & 45.8 & 47.5 \\
\textbf{OPDVR (Ours)} & \textbf{36.9} & \textbf{28.1} & \textbf{64.8} & 84.7 & 33.2 & \textbf{47.0} & \textbf{49.1} \\
\bottomrule
\end{tabular}%
}
\end{table}

\begin{table}[t]
\centering
\caption{Results on cross-architecture distillation (Qwen3-1.7B-Base $\leftarrow$ Qwen3-4B-Base-RL). All models are evaluated with avg@16 accuracy.}
\label{tab:cross_arch}
\resizebox{\textwidth}{!}{%
\begin{tabular}{l|cccccc|c}
\toprule
\textbf{Method} & \textbf{AIME24} & \textbf{AIME25} & \textbf{AMC} & \textbf{MATH500} & \textbf{Minerva} & \textbf{OlympiadBench} & \textbf{Avg.} \\
\midrule
Student (Qwen3-1.7B-Base) & 4.1 & 1.7 & 23.2 & 48.9 & 8.9 & 17.1 & 17.3 \\
Teacher (Qwen3-4B-Base-RL) & 10.6 & 13.1 & 40.3 & 74.2 & 17.2 & 30.0 & 30.9 \\
\midrule
Sampled-Token OPD & 6.5 & 2.1 & 24.8 & 59.1 & 11.5 & 21.6 & 20.9 \\
Top-64 OPD & \textbf{8.5} & \textbf{3.3} & 26.4 & 60.1 & 10.7 & 21.4 & 21.7 \\
OPD+GRPO & 8.3 & 1.7 & 23.2 & 58.0 & 10.2 & 21.8 & 20.5 \\
RLSD & 6.6 & 2.9 & 28.7 & 59.3 & 10.4 & 21.6 & 21.6 \\
ExOPD & 7.0 & \textbf{3.3} & 25.9 & 60.4& \textbf{12.3} & 21.4 & 21.6 \\
SG-OPD & 5.5 & 2.8 & 28.2 & 59.7 & 10.7 & 22.3 & 21.5 \\
\textbf{OPDVR (Ours)} & \textbf{8.5} & \textbf{3.3} & \textbf{30.3} & \textbf{60.8} & 11.6 & \textbf{22.0} & \textbf{22.8} \\
\bottomrule
\end{tabular}%
}
\end{table}

Tables~\ref{tab:same_arch} and \ref{tab:cross_arch} report the main results on six reasoning benchmarks. As shown in both tables, OPDVR consistently outperforms all baselines in terms of average accuracy across both distillation settings. In the same-architecture setting, OPDVR achieves the highest average score of 49.1, outperforming the next best method (ExOPD, 47.9) by 1.2 points, and even surpasses the teacher model on AIME24 (36.9 vs. 36.0). In the cross-architecture setting, OPDVR again achieves the best average score of 22.8, outperforming the strongest baseline (Top-64 OPD, 21.7) by 1.1 points, with notable gains on AMC (30.3 vs. 28.7 for RLSD) and OlympiadBench (22.0 vs. 21.8 for OPD+GRPO), highlighting its robustness to both architectural differences and diverse baseline designs.

\subsection{Group Relative Policy Distillation}

To further validate the effectiveness of replacing the binary verifier reward with group-relative advantages, we instantiate OPDVR with GRPO-style advantage estimation (describe in Section~\ref{sec:GRPD}) and evaluate it under the same-architecture setting. The teacher model is the same Qwen3-4B-Nonthinking model trained with GRPO on DeepMath. The student model is Qwen3-4B-nonthinking. To more cleanly evaluate the effectiveness of replacing binary rewards with group-relative advantages, we conduct this experiment on DAPO-Math-17K, distinct from the teacher's DeepMath training data. The group size is set to \(G=8\). We compare GRPD against GRPO, standard sampled-token OPD and Distilled RL~\citep{wang2026distilled}, a method that selectively applies OPD on positive advantages and GRPO on negative advantages.

Table~\ref{tab:grpd} reports the results. GRPD consistently outperforms all baselines across all six benchmarks. It achieves notable gains of 6.5 points on AIME24 and 10.9 points on AIME25 over GRPO, and surpasses Distilled RL by 1.7 points on AIME24 and 0.7 points on average, demonstrating that a unified group-relative advantage formulation yields stronger and more stable improvements than heuristic advantage-based switching between OPD and RLVR objectives.
\begin{table}[t]
\centering
\caption{Results on Group Relative Policy Distillation (Qwen3-4B $\leftarrow$ Qwen3-4B-RL). All models are evaluated with avg@16 accuracy.}
\label{tab:grpd}
\resizebox{\textwidth}{!}{%
\begin{tabular}{l|cccccc|c}
\toprule
\textbf{Method} & \textbf{AIME24} & \textbf{AIME25} & \textbf{AMC} & \textbf{MATH500} & \textbf{Minerva} & \textbf{OlympiadBench} & \textbf{Avg.} \\
\midrule
Student (Qwen3-4B)  & 24.0 & 15.8 & 60.8 & 80.9 & 27.6 & 42.9 & 42.0 \\
Teacher (Qwen3-4B-RL) & 36.0 & 29.0 & 65.9 & 87.0 & 35.4 & 49.3 & 50.4 \\
\midrule
GRPO & 28.3 & 20.8 & 62.3 & 83.9 & 28.9 & 44.6 & 44.8 \\
OPD & 32.0 & \textbf{31.7} & 65.6 & 85.4 & 28.9 & 46.6 & 48.4 \\
Distilled RL& 33.1 & 28.9 & 66.6 & \textbf{86.3} & 30.3 & 46.8 & 48.7 \\
\textbf{GRPD (Ours)} & \textbf{34.8} & \textbf{31.7} & \textbf{67.0} & 85.6 & \textbf{30.5} & \textbf{47.0} & \textbf{49.4} \\
\bottomrule
\end{tabular}%
}
\end{table}

\subsection{Ablation Study: Inverse-Gated Experiment}
\begin{figure}[t]
    \centering
    \begin{minipage}[c]{0.52\textwidth}
        \centering
        \includegraphics[width=\textwidth]{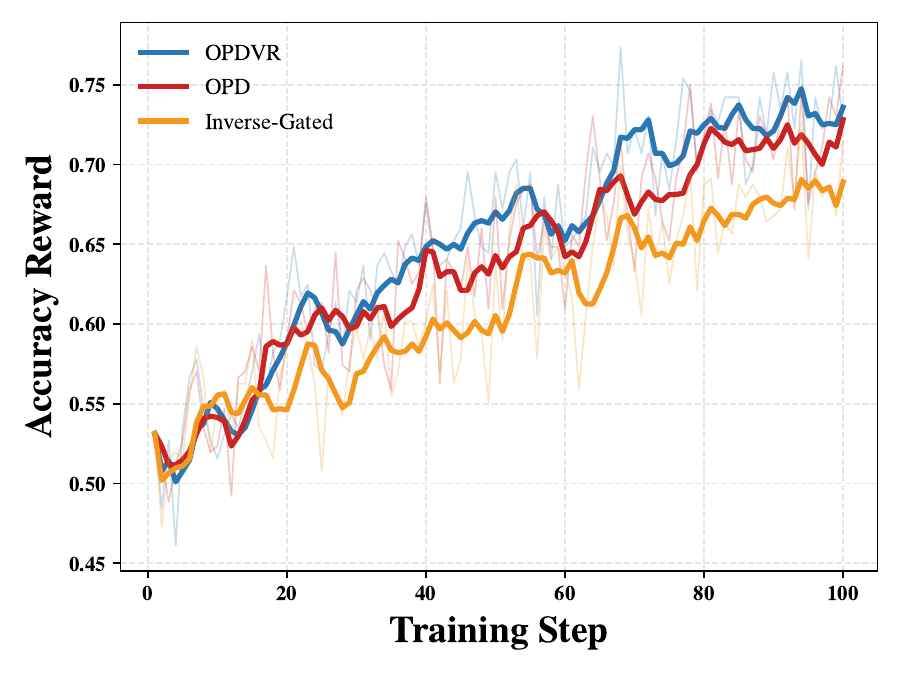}
    \end{minipage}
    \hfill
    \begin{minipage}[c]{0.47\textwidth}
        \centering
        \includegraphics[width=\textwidth]{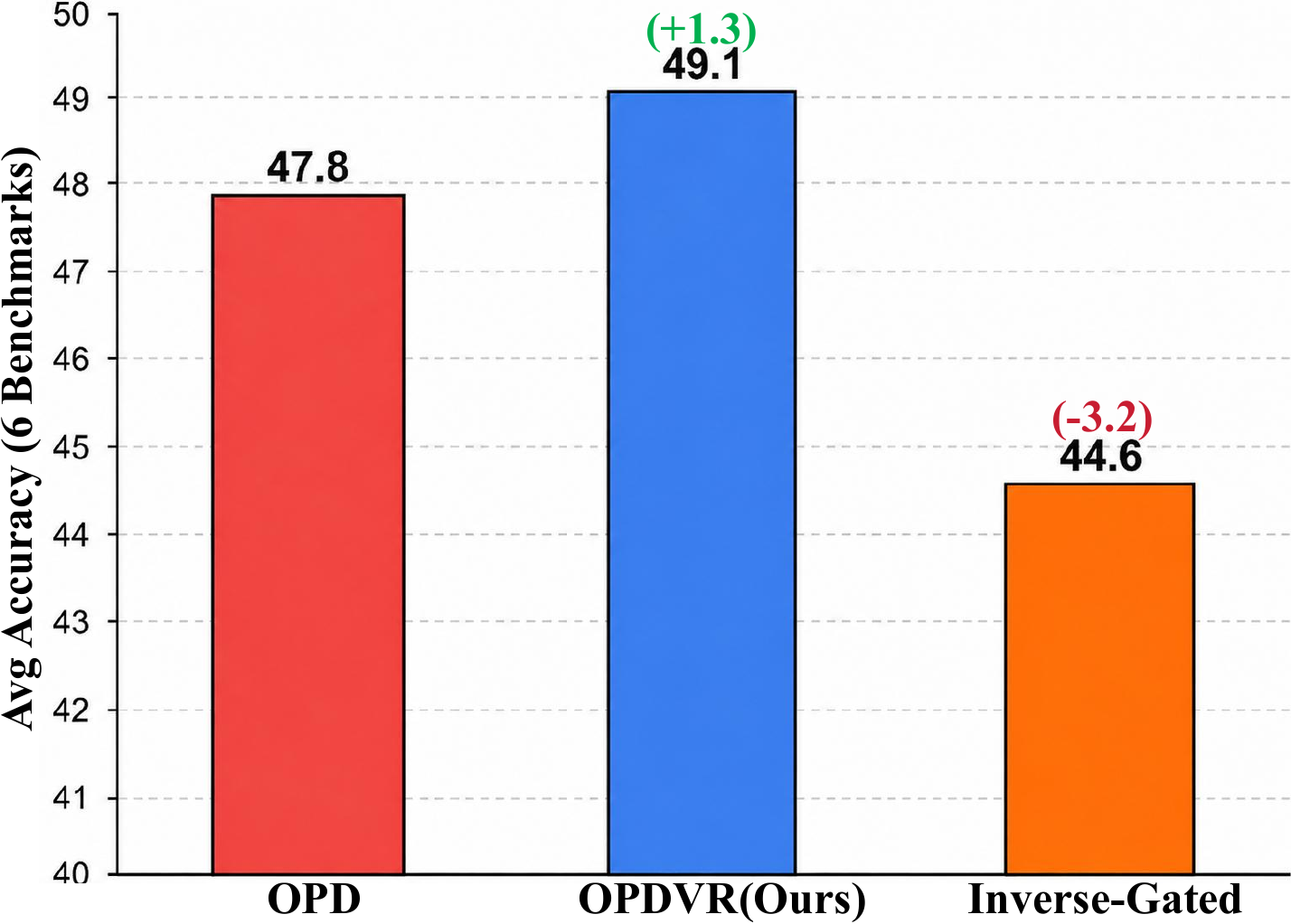}
    \end{minipage}
    \caption{Ablation study on the gating mechanism. \textbf{Left}: training-time accuracy reward of OPD, OPDVR, and the inverse-gated variant. \textbf{Right}: average accuracy over six benchmarks.}
    \label{fig:ablation_reward}
\end{figure}

\begin{table}[t]
\centering
\caption{Ablation study on the gating mechanism. We compare OPD, OPDVR, and an inverse-gated variant where the ReLU gate is applied in the opposite direction.}
\label{tab:ablation}
\resizebox{\textwidth}{!}{%
\begin{tabular}{l|cccccc|c}
\toprule
\textbf{Method} & \textbf{AIME24} & \textbf{AIME25} & \textbf{AMC} & \textbf{MATH500} & \textbf{Minerva} & \textbf{OlympiadBench} & \textbf{Avg.} \\
\midrule
Student (Qwen3-4B) & 24.0 & 15.8 & 60.8 & 80.9 & 27.6 & 42.9 & 42.0 \\
Teacher (Qwen3-4B-RL) & 36.0 & 29.0 & 65.9 & 87.0 & 35.4 & 49.3 & 50.4 \\
\midrule
OPD & 34.2 & 26.0 & 63.1 & 85.5 & 31.6 & 46.5 & 47.8 \\
OPDVR (Ours) & 36.9 & 28.1 & 64.8 & 84.7 & 33.2 & 47.0 & 49.1 \\
Inverse-Gated & 30.3 & 21.2 & 62.3 & 83.6 & 27.7 & 42.8 & 44.6 \\
\bottomrule
\end{tabular}%
}
\end{table}
% \begin{figure}[t]
%     \centering
%     \begin{minipage}[c]{0.50\textwidth}
%         \centering
%         \includegraphics[width=\textwidth]{figs/ablation_inverse_true_reward.pdf}
%     \end{minipage}
%     \hfill
%     \begin{minipage}[c]{0.47\textwidth}
%         \centering
%         \includegraphics[width=\textwidth]{figs/ablation_avg_bar.pdf}
%     \end{minipage}
%     \caption{Ablation study on the gating mechanism. \textbf{Left}: training-time accuracy reward of OPD, OPDVR, and the inverse-gated variant. \textbf{Right}: average accuracy over six benchmarks.}
%     \label{fig:ablation_reward}
% \end{figure}
To examine the effectiveness of the gating mechanism in OPDVR, we conduct an inverse-gated ablation under the same-architecture setting (Qwen3-4B $\leftarrow$ Qwen3-4B-RL) on DeepMath, following the same experimental setup as the main experiments.. Recall that OPDVR keeps a token only when its update direction agrees with the verifier signal: on a correct trajectory ($R{=}+1$), tokens where the teacher is more confident than the student ($\pi_T > \pi_\theta$) are kept and tokens where the student is already more confident ($\pi_\theta > \pi_T$) are gated out; on an incorrect trajectory ($R{=}-1$), the roles are reversed -- tokens where the student is more confident than the teacher are kept, and tokens where the teacher is more confident are gated out. The inverse-gated variant swaps these two sets exactly: it keeps the tokens that OPDVR gates out (the student-more-confident tokens on correct trajectories and the teacher-more-confident tokens on incorrect trajectories), and gates out the tokens that OPDVR keeps, while keeping the masking ratio and everything else unchanged.

Table~\ref{tab:ablation} shows the final benchmark results, and Figure~\ref{fig:ablation_reward} tracks the accuracy reward on the training set. The three variants start from the same point and separate monotonically throughout training, yielding a consistent ordering OPDVR $>$ OPD $>$ Inverse-Gated in both the training curves and the final benchmarks. Reversing the gate falls below vanilla OPD on all six benchmarks, confirming the effectiveness of the gating mechanism. Notably, even the inverse-gated variant still improves over the initial model, indicating that the teacher's distributional guidance remains useful—though its benefits are substantially hindered when the sign is misaligned with trajectory correctness
\begin{figure}[t]
    \centering
    \begin{subfigure}[b]{0.9\textwidth}
        \centering
        \includegraphics[width=\linewidth]{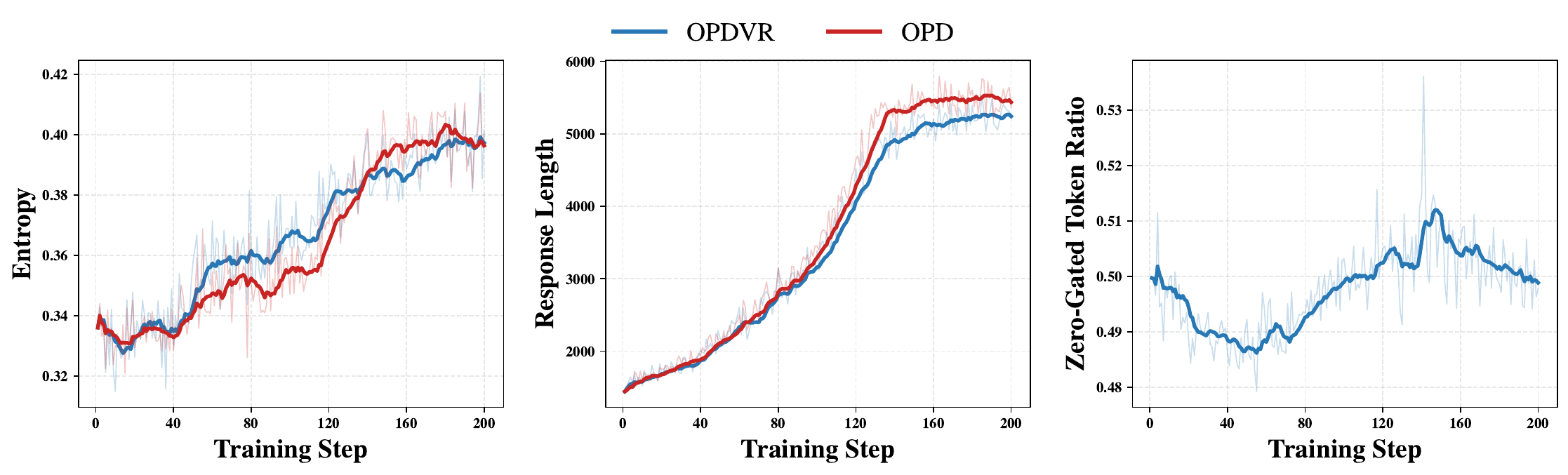}
        \caption{Training dynamic of same-architecture distillation (Qwen3-4B~$\leftarrow$~Qwen3-4B-RL).}
    \end{subfigure}    
    
    \begin{subfigure}[b]{0.9\textwidth}
        \centering
        \includegraphics[width=\linewidth]{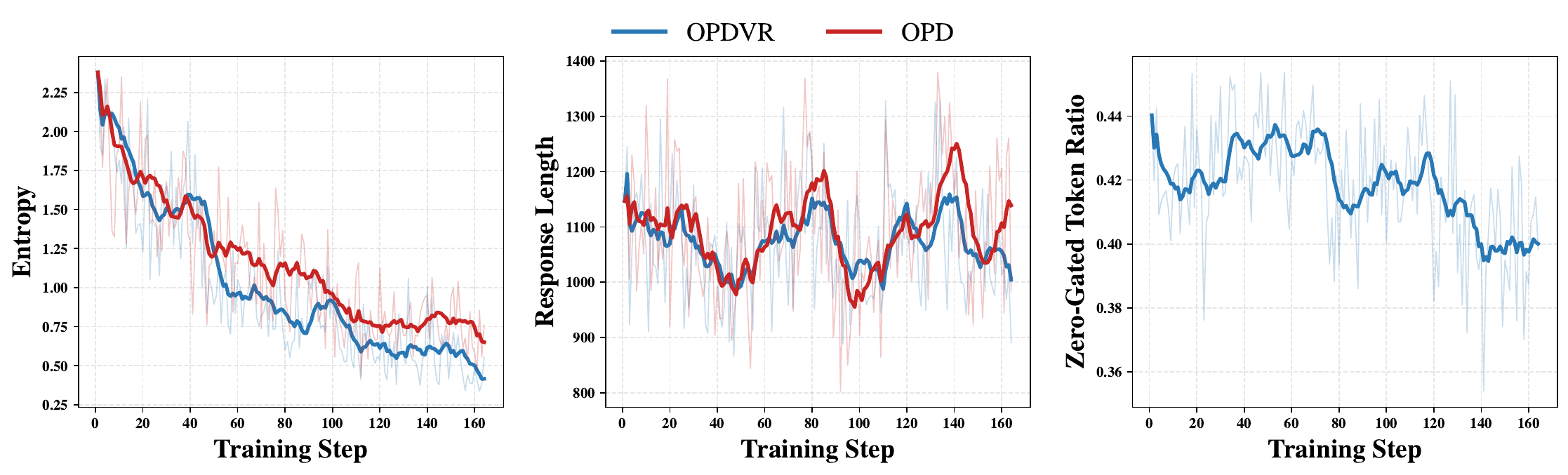}
        \caption{Training dynamic of cross-architecture distillation (Qwen3-1.7B-Base~$\leftarrow$~Qwen3-4B-Base-RL).}
    \end{subfigure}
    
    \caption{Training dynamics}
    \label{fig:train_dynamics}
\end{figure}

\subsection{Training Dynamics}

% \begin{figure}[t]
%     \centering
%     \includegraphics[width=1.0\textwidth]{figs/training_dynamic_4b.pdf}
%     \caption{Training dynamics same-architecture distillation (Qwen3-4B~$\leftarrow$~Qwen3-4B-RL).}
%     \label{fig:train_4b}
% \end{figure}
% \begin{figure}[t]
%     \centering
%     \includegraphics[width=1.0\textwidth]{figs/training_dynamic_1.7b.pdf}
%     \caption{Training dynamics of cross-architecture distillation (Qwen3-1.7B-Base~$\leftarrow$~Qwen3-4B-Base-RL).}
%     \label{fig:train_1.7b}
% \end{figure}

To understand how token-level gating shapes the distillation process, we track three quantities throughout training: the student policy entropy, the average response length, and the zero-gated token ratio, i.e., the fraction of tokens whose distillation loss is masked by the gate. Figure~\ref{fig:train_dynamics} shows the dynamics of the same-architecture setting (Qwen3-4B~$\leftarrow$~Qwen3-4B-RL) and the cross-architecture setting (Qwen3-1.7B-Base~$\leftarrow$~Qwen3-4B-Base-RL), respectively.

The entropy and response-length trajectories of the sampled-token OPD baseline show no consistent pattern across settings; instead, they depend strongly on the teacher--student pair. In the same-architecture setting, the entropy of both methods drifts mildly upward (from $\sim$0.33 to $\sim$0.40) while the response length inflates by more than four times, from roughly 1.6k to over 6.7k tokens. In the cross-architecture setting, the picture inverts: the student entropy collapses rapidly from $\sim$2.0 at the start of training and the response length stays nearly flat throughout. These quantities therefore do not follow a setting-independent trend -- their evolution is dictated by the specific teacher and student models rather than by the distillation objective itself.

In contrast, the zero-gated token ratio is consistent across both settings. It remains in a band around fifty percent during the entire course of training ($\approx$0.48--0.50 for the 4B student and $\approx$0.40--0.44 for the 1.7B student), never degenerating toward the trivial extremes of gating all or no tokens. This means that roughly half of the sampled tokens are zeroed by the ReLU gate at any point of training. Recalling the failure cases identified in Section~\ref{sec:Method}, these are exactly the tokens whose implicit OPD reward pushes against the RLVR direction: on correct trajectories where the student already assigns higher probability than the teacher ($\pi_\theta > \pi_T$), and on incorrect trajectories where the teacher still outranks the student ($\pi_T > \pi_\theta$). The stable ratio indicates that such RLVR-violating tokens constitute a persistent fraction of the data throughout training -- the gate consistently filters them out while distilling the remaining tokens with their full teacher--student log-ratio magnitude. 

\section{Conclusion}
We revisited sampled-token OPD from an RLVR perspective and observed that its implicit reward is governed by the teacher-student probability ratio rather than trajectory correctness, which can lead to updates that can penalize tokens in correct trajectories or reward tokens in incorrect ones. To address this, we proposed OPDVR, which adds a simple ReLU gate on the sampled token to enforce RLVR compliance. This minimal modification combines OPD and RLVR without any hyperparameter or heuristic trade-off. Experiments across mathematical reasoning benchmarks demonstrate that OPDVR consistently outperforms standard OPD. 

More importantly, by reformulating OPD as an RLVR method, OPDVR becomes a general framework that can be readily integrated with any policy gradient algorithm, including REINFORCE, GRPO, and DAPO. We instantiate this with GRPO and present Group Relative Policy Distillation (GRPD), which further demonstrates the versatility and strength of our approach. The empirical success of both OPDVR and GRPD across multiple benchmarks confirms that aligning distillation signals with trajectory correctness is not only effective but also broadly applicable, offering a simple yet powerful recipe for future post-training methods.

% \FloatBarrier

\clearpage
\bibliography{main}
\bibliographystyle{icml2025}

\appendix

\section{Theoretical Analysis}
\label{sec:theory}

On-policy distillation supervises student-generated prefixes, and OPDVR adds a
verifier gate to this feedback. We establish joint improvement of task success
and teacher matching for the population mirror update of OPDVR.

\paragraph{Setup.}
Consider a prefix process of horizon $H$ with finite prompt and action sets,
a fixed prompt distribution, and terminal verifier $Y\in\{0,1\}$. Each state
includes the time index. Fix student $p$ and teacher $T$, both with full
support. Define $Q^p(s,a)=\Pr_p(Y=1\mid s,a)$, where all subsequent actions
follow $p$, and $J(p)=\Pr_p(Y=1)$. With
$z(s,a)=\log[T(a\mid s)/p(a\mid s)]$, the gate is
$c(Y,z)=Y[z]_+-(1-Y)[-z]_+$, where $[x]_+=\max\{x,0\}$. Its conditional mean
$\bar c(s,a)=\mathbb E[c(Y,z(s,a))\mid s,a]$ is
\begin{equation}
\label{eq:gate}
\bar c(s,a)=Q^p(s,a)[z(s,a)]_+-\bigl(1-Q^p(s,a)\bigr)[-z(s,a)]_+.
\end{equation}

\begin{assumption}[Soft-value teacher]
\label{asm:teacher}
At every reachable prefix $s$, there exists $\beta_s>0$ such that
\begin{equation}
\label{eq:teacher}
T(a\mid s)=\frac{p(a\mid s)e^{\beta_sQ^p(s,a)}}{Z_s},
\qquad Z_s=\mathbb E_{a\sim p(\cdot\mid s)}e^{\beta_sQ^p(s,a)}.
\end{equation}
\end{assumption}

\noindent
This is the form of a KL-regularized policy improvement
\citep{geist2019regularized}. It is a condition on the current pair $(p,T)$
and is not implied by higher aggregate teacher accuracy.

\paragraph{Mirror update.}
Apply the standard KL mirror step \citep[\S5]{geist2019regularized}
simultaneously at all reachable prefixes,
\begin{equation}
\label{eq:update}
p^+(a\mid s)=
\frac{p(a\mid s)e^{\eta\bar c(s,a)}}
{\sum_{a'}p(a'\mid s)e^{\eta\bar c(s,a')}},
\end{equation}
holding the teacher and the pre-update quantities $Q^p,z,\bar c$ fixed. The
next two results characterize this update and bound its effect on the two
objectives.

\begin{proposition}[KL-regularized characterization]
\label{prop:char}
Fix a reachable prefix $s$. Under Assumption~\ref{asm:teacher}, the teacher
$T(\cdot\mid s)$ and the updated policy $p^+(\cdot\mid s)$ in Equation~\ref{eq:update}
are the unique maximizers over action distributions $r$,
\begin{align}
T(\cdot\mid s)
&=\argmax_{r}\Bigl[\mathbb E_{r}Q^p(s,\cdot)
  -\beta_s^{-1}D_{\mathrm{KL}}(r\Vert p(\cdot\mid s))\Bigr],
\label{eq:char-teacher}\\
p^+(\cdot\mid s)
&=\argmax_{r}\Bigl[\mathbb E_{r}\bar c(s,\cdot)
  -\eta^{-1}D_{\mathrm{KL}}(r\Vert p(\cdot\mid s))\Bigr].
\label{eq:char-update}
\end{align}
\end{proposition}

\begin{theorem}[Joint improvement]
\label{thm:joint}
Under Assumption~\ref{asm:teacher}, the mirror update Equation~\ref{eq:update}
satisfies:
\begin{itemize}
\item[(i)] for every $\eta>0$, task success does not decrease,
\begin{equation}
\label{eq:joint-i}
J(p^+)\ge J(p);
\end{equation}
\item[(ii)] if in addition $\eta\le1/2$, teacher KL does not increase at any
reachable prefix $s$,
\begin{equation}
\label{eq:joint-ii}
D_{\mathrm{KL}}\bigl(p^+(\cdot\mid s)\Vert T(\cdot\mid s)\bigr)
\le D_{\mathrm{KL}}\bigl(p(\cdot\mid s)\Vert T(\cdot\mid s)\bigr).
\end{equation}
\end{itemize}
If $Q^p(s,\cdot)$ is nonconstant at a prefix visited with positive
probability under $p$, the inequality in~(i) is strict, and that in~(ii) is
strict at that prefix.
\end{theorem}

A single verifier-gated on-policy step advances task accuracy and teacher
matching together, not one at the expense of the other. Accuracy improves at
any step size, and teacher matching is preserved for a modest step. This is
the behaviour OPDVR targets. Proofs are in Appendix~\ref{sec:proofs}.

\section{Proofs}
\label{sec:proofs}

Equation~\ref{eq:gate} is immediate from $\mathbb E[Y\mid s,a]=Q^p(s,a)$. Because
$z(s_t,a_t)$ and the score $\psi_t=\nabla_\theta\log p_\theta(a_t\mid s_t)$ of
any smooth parameterization ($p=p_{\theta_0}$) are determined by $(s_t,a_t)$,
the tower property gives
\[
\mathbb E_{\tau\sim p}\!\Bigl[\sum_t c(Y,z(s_t,a_t))\,\psi_t\Bigr]
=\mathbb E_{\tau\sim p}\!\Bigl[\sum_t \bar c(s_t,a_t)\,\psi_t\Bigr],
\]
so the population signal analysed below is exactly the one the sampled update
follows in expectation; $Q^p$ enters only through the analysis.

\subsection{Proof of Proposition~\ref{prop:char}}
\label{sec:pf-char}

Fix a reachable prefix $s$ and use the local notation $p(a)=p(a\mid s)$,
$T(a)=T(a\mid s)$, $Q(a)=Q^p(s,a)$, $z(a)=z(s,a)$, $\bar c(a)=\bar c(s,a)$,
$p^+(a)=p^+(a\mid s)$, $\beta=\beta_s$, and $Z=Z_s$; this suppresses only the
fixed prefix, and full support with finiteness makes every sum finite and
every KL divergence well-defined. The continuation policy in $Q$ remains the
pre-update student. Assumption~\ref{asm:teacher} rearranges to
$Q(a)=\beta^{-1}\log[T(a)/p(a)]+\beta^{-1}\log Z$, so that for every action
distribution $r$,
\begin{equation}
\label{eq:a2}
\begin{aligned}
\mathbb E_r Q-\beta^{-1}D_{\mathrm{KL}}(r\Vert p)
&=\sum_a r(a)Q(a)-\beta^{-1}\sum_a r(a)\log\frac{r(a)}{p(a)}\\
&=\beta^{-1}\sum_a r(a)\!\left(\log\frac{T(a)}{p(a)}+\log Z\right)
  -\beta^{-1}\sum_a r(a)\log\frac{r(a)}{p(a)}\\
&=\beta^{-1}\log Z+\beta^{-1}\sum_a r(a)\log\frac{T(a)}{r(a)}\\
&=\beta^{-1}\log Z-\beta^{-1}D_{\mathrm{KL}}(r\Vert T).
\end{aligned}
\end{equation}
Since $D_{\mathrm{KL}}(r\Vert T)\ge0$ with equality iff $r=T$, the teacher $T$
uniquely maximizes $\mathbb E_rQ-\beta^{-1}D_{\mathrm{KL}}(r\Vert p)$, which is
Equation~\ref{eq:char-teacher}. Next define
$W(u)=\sum_a p(a)e^{u\bar c(a)}$ for $u\ge0$. The mirror update rearranges to
$\bar c(a)=\eta^{-1}\log[p^+(a)/p(a)]+\eta^{-1}\log W(\eta)$, so that the same
computation with $(T,Z,\beta)$ replaced by $(p^+,W(\eta),\eta)$ gives
\begin{equation}
\label{eq:a3}
\begin{aligned}
\mathbb E_r \bar c-\eta^{-1}D_{\mathrm{KL}}(r\Vert p)
&=\sum_a r(a)\bar c(a)-\eta^{-1}\sum_a r(a)\log\frac{r(a)}{p(a)}\\
&=\eta^{-1}\sum_a r(a)\!\left(\log\frac{p^+(a)}{p(a)}+\log W(\eta)\right)
  -\eta^{-1}\sum_a r(a)\log\frac{r(a)}{p(a)}\\
&=\eta^{-1}\log W(\eta)+\eta^{-1}\sum_a r(a)\log\frac{p^+(a)}{r(a)}\\
&=\eta^{-1}\log W(\eta)-\eta^{-1}D_{\mathrm{KL}}(r\Vert p^+).
\end{aligned}
\end{equation}
Hence $p^+$ uniquely maximizes $\mathbb E_r\bar c-\eta^{-1}D_{\mathrm{KL}}(r\Vert p)$,
which is Equation~\ref{eq:char-update}; full support makes both maximizers
well-defined.

\subsection{Proof of Theorem~\ref{thm:joint}}
\label{sec:pf-joint}

All quantities $p,T,Q^p,z,\bar c$ are fixed at the current iterate. Set
$V^p(s)=\mathbb E_{a\sim p(\cdot\mid s)}Q^p(s,a)$. Early termination is padded with
absorbing transitions up to horizon $H$; at terminal states,
$V^p(s_{H+1})=Y(s_{H+1})$. All logarithms are natural, and finiteness and full
support justify the finite sums and differentiation used below. We reuse the
local prefix notation of Appendix~\ref{sec:pf-char}; a monotonicity property
and an interpolation path serve both parts~(i) and~(ii).

\paragraph{Monotonicity of the gated coefficient.}
With the local prefix notation, set
\begin{equation}
\label{eq:a4}
b=\frac{\log\mathbb E_p e^{\beta Q}}{\beta},
\qquad z(a)=\beta\bigl(Q(a)-b\bigr).
\end{equation}
Since $Q\in[0,1]$, we have $b\in[0,1]$. Substituting
Equation~\ref{eq:a4} into the gate gives $\bar c(a)=g(Q(a))$, where
\begin{equation}
\label{eq:a5}
g(q)=
\begin{cases}
\beta(1-q)(q-b),&q<b,\\
\beta q(q-b),&q\ge b.
\end{cases}
\end{equation}
The function is continuous at $b$. Away from this point,
\begin{equation}
\label{eq:a6}
g'(q)=
\begin{cases}
\beta(1+b-2q),&q<b,\\
\beta(2q-b),&q>b.
\end{cases}
\qquad 0\le g'(q)\le 2\beta.
\end{equation}
Thus $g$ is nondecreasing. When $\bar c$ is viewed as a function of
$z=\beta(q-b)$, its derivative lies in $[0,2]$. Continuity at the join point
extends both properties across it, so $\bar c$ is a nondecreasing,
$2$-Lipschitz function of $z$ on the relevant interval.

For any functions $f,h$ and distribution $\mu$ on the action set, with
independent $a,a'\sim\mu$, expansion gives the pairwise covariance identity
\begin{equation}
\label{eq:a7}
\operatorname{Cov}_\mu(f,h)
=\frac12\mathbb E_{a,a'\sim\mu}
\bigl[(f(a)-f(a'))(h(a)-h(a'))\bigr].
\end{equation}
Monotonicity of $g$ makes $(Q(a)-Q(a'))(\bar c(a)-\bar c(a'))$ nonnegative.
Monotonicity and the Lipschitz bound in $z$ also imply
\begin{equation}
\label{eq:a8}
(z(a)-z(a'))(\bar c(a)-\bar c(a'))
\ge\frac12\bigl(\bar c(a)-\bar c(a')\bigr)^2.
\end{equation}
Applying Equation~\ref{eq:a7} to both products yields, for every $\mu$,
\begin{equation}
\label{eq:a9}
\operatorname{Cov}_\mu(Q,\bar c)\ge0,
\qquad
\operatorname{Cov}_\mu(z,\bar c)\ge\frac12\operatorname{Var}_\mu(\bar c).
\end{equation}

\paragraph{The interpolation path.}
For the fixed prefix, introduce the interpolation parameter $u$, distinct
from generation time $t$, and define
\begin{equation}
\label{eq:a10}
p_u(a)=\frac{p(a)e^{u\bar c(a)}}{W(u)},\qquad 0\le u\le\eta,
\qquad p_0=p,\quad p_\eta=p^+.
\end{equation}
The quantities $Q,z,\bar c,T$ are held fixed along this path, and $p_u$ has
full support for every finite $u$. Since
$\log p_u(a)=\log p(a)+u\bar c(a)-\log W(u)$ and
$\partial_u\log W(u)=W'(u)/W(u)=\sum_a p_u(a)\bar c(a)=\mathbb E_{p_u}\bar c$,
\begin{equation}
\label{eq:score}
\partial_u\log p_u(a)=\bar c(a)-\partial_u\log W(u)=\bar c(a)-\mathbb E_{p_u}\bar c.
\end{equation}
Using $\partial_u p_u(a)=p_u(a)\,\partial_u\log p_u(a)$, every fixed function
$f$ satisfies
\begin{equation}
\label{eq:a11}
\begin{aligned}
\frac{d}{du}\mathbb E_{p_u}f
&=\sum_a f(a)\,\partial_u p_u(a)
 =\sum_a f(a)\,p_u(a)\bigl(\bar c(a)-\mathbb E_{p_u}\bar c\bigr)\\
&=\mathbb E_{p_u}(f\bar c)-\mathbb E_{p_u}f\,\mathbb E_{p_u}\bar c
 =\operatorname{Cov}_{p_u}(f,\bar c).
\end{aligned}
\end{equation}
If $Q$ is constant, Assumption~\ref{asm:teacher} gives $T=p$,
$\bar c=0$, and $p^+=p$, so parts~(i) and~(ii) hold with equality; otherwise
$\min_a Q(a)<b<\max_a Q(a)$ forces $0<b<1$, the derivative in
Equation~\ref{eq:a6} is then strictly positive on each interval, so $g$ is
strictly increasing, and every $p_u$ has full support, giving
$\operatorname{Cov}_{p_u}(Q,\bar c)>0$ and $\operatorname{Var}_{p_u}(\bar c)>0$.

\paragraph{(i) Task success does not decrease.}
Taking $f=Q$ in Equation~\ref{eq:a11} and applying Equation~\ref{eq:a9}
give $\frac{d}{du}\mathbb E_{p_u}Q=\operatorname{Cov}_{p_u}(Q,\bar c)\ge0$, so
$u\mapsto\mathbb E_{p_u}Q$ is nondecreasing on $[0,\eta]$; evaluating at the
endpoints $u=0$ and $u=\eta$ yields
\begin{equation}
\label{eq:a12}
\mathbb E_{p^+(\cdot\mid s)}Q^p(s,\cdot)
\ge
\mathbb E_{p(\cdot\mid s)}Q^p(s,\cdot)
=V^p(s).
\end{equation}
Return to the full prefix notation. Let $d_t^r$ be the distribution of $s_t$
under policy $r$, with the fixed prompt distribution. The terminal
convention above and the Bellman relation give
$Q^p(s_t,a_t)=\mathbb E[V^p(s_{t+1})\mid s_t,a_t]$, so under $\tau\sim p^+$ each
$Q^p(s_t,a_t)$ may be replaced by $V^p(s_{t+1})$ in expectation; telescoping
then gives the finite-horizon performance-difference identity
(\citealp{kakade2002approximately}; see also \citealp[\S2]{schulman2015trust}),
\begin{equation}
\label{eq:a15}
\begin{aligned}
&\mathbb E_{\tau\sim p^+}\sum_{t=1}^{H}
\bigl(Q^p(s_t,a_t)-V^p(s_t)\bigr)\\
&\quad=\mathbb E_{\tau\sim p^+}\sum_{t=1}^{H}
\bigl(V^p(s_{t+1})-V^p(s_t)\bigr)\\
&\quad=\mathbb E_{\tau\sim p^+}\bigl(V^p(s_{H+1})-V^p(s_1)\bigr)
=J(p^+)-J(p).
\end{aligned}
\end{equation}
The last equality uses the terminal convention
$\mathbb E_{\tau\sim p^+}V^p(s_{H+1})=\mathbb E_{\tau\sim p^+}Y=J(p^+)$, and the fact that
$s_1$ follows the fixed prompt distribution independently of the policy, so
$\mathbb E_{\tau\sim p^+}V^p(s_1)=\mathbb E_{s_1}V^p(s_1)=J(p)$. Conditioning each summand
on $s_t$ yields
\begin{equation}
\label{eq:a16}
\begin{aligned}
&J(p^+)-J(p)\\
&\quad=\sum_{t=1}^{H}\mathbb E_{s\sim d_t^{p^+}}
\left[
\mathbb E_{a\sim p^+(\cdot\mid s)}Q^p(s,a)-V^p(s)
\right].
\end{aligned}
\end{equation}
All terms inside the expectation are nonnegative by Equation~\ref{eq:a12},
proving $J(p^+)\ge J(p)$ for every $\eta>0$; the distribution here is the \emph{new} visitation
distribution $d_t^{p^+}$, so the conclusion concerns actual success after the
simultaneous update. Moreover $p^+$ has full support by
Equation~\ref{eq:update}, so every prefix visited with positive probability
under $p$ is also visited under $p^+$; if $Q^p(s,\cdot)$ is nonconstant at
such a prefix, Equation~\ref{eq:a12} is strict there, the corresponding
bracket in Equation~\ref{eq:a16} is strictly positive, and hence
$J(p^+)>J(p)$.

\paragraph{(ii) Teacher KL does not increase at any reachable prefix.}
Let $K_s(u)=D_{\mathrm{KL}}(p_u\Vert T)=\mathbb E_{p_u}\log\frac{p_u}{T}$. Since
$T(a)=p(a)e^{z(a)}$, the log-density ratio is
\begin{equation}
\label{eq:a13}
\log\frac{p_u(a)}{T(a)}
=\log\frac{p(a)e^{u\bar c(a)}/W(u)}{p(a)e^{z(a)}}
=-z(a)+u\bar c(a)-\log W(u).
\end{equation}
Differentiating $K_s$ splits into a measure term and an integrand term; by
Equation~\ref{eq:a11} the former is a covariance with $\bar c$, and by
Equation~\ref{eq:score} the latter,
$\mathbb E_{p_u}\!\left[\partial_u\log\frac{p_u}{T}\right]
=\mathbb E_{p_u}\bigl(\bar c-\mathbb E_{p_u}\bar c\bigr)=0$, vanishes. Using
Equation~\ref{eq:a13} and dropping the action-independent term $\log W(u)$,
\begin{equation}
\label{eq:a14}
\begin{aligned}
K_s'(u)
&=\operatorname{Cov}_{p_u}\!\left(\log\frac{p_u}{T},\bar c\right)
  +\mathbb E_{p_u}\!\left[\partial_u\log\frac{p_u}{T}\right]
 =\operatorname{Cov}_{p_u}\!\left(-z+u\bar c,\bar c\right)\\
&=u\operatorname{Var}_{p_u}(\bar c)-\operatorname{Cov}_{p_u}(z,\bar c)
 \le\left(u-\tfrac12\right)\operatorname{Var}_{p_u}(\bar c),
\end{aligned}
\end{equation}
the last step by Equation~\ref{eq:a9}. For $0<\eta\le1/2$ the derivative is
nonpositive throughout $[0,\eta]$, so $K_s(\eta)\le K_s(0)$, that is,
$D_{\mathrm{KL}}\bigl(p^+(\cdot\mid s)\Vert T(\cdot\mid s)\bigr)\le
D_{\mathrm{KL}}\bigl(p(\cdot\mid s)\Vert T(\cdot\mid s)\bigr)$. If $Q$ is
nonconstant, then $\operatorname{Var}_{p_u}(\bar c)>0$, so $K_s'(u)<0$ for all
$0\le u<\eta$ and the decrease is strict, including when $\eta=1/2$. This
completes the proof of both parts of Theorem~\ref{thm:joint}. $\square$

\section{Directional Alignment and Conflict Removal}
\label{app:opdvr-theory}

We now provide a formal characterization of the advantage of OPDVR over the sampled-token OPD objective. Throughout this section, let $s_t=(q,o_{<t})$ denote the state at step $t$, $a_t=o_t$ the sampled action, and define the token-level teacher-student log-ratio as
\begin{equation}
r_t = \log \frac{\pi_T(a_t \mid s_t)}{\pi_\theta(a_t \mid s_t)}.
\end{equation}
As before, $R \in \{+1,-1\}$ denotes the trajectory-level verifier reward.

\subsection{Directional Alignment with the Verifier Gradient}

Let $u_t = \nabla_\theta \log \pi_\theta(a_t \mid s_t)$ denote the policy gradient direction of the sampled token. The token-level update direction of sampled-token OPD is
\begin{equation}
\Delta_{\mathrm{OPD},t} = r_t u_t,
\label{eq:opd-update}
\end{equation}
whereas the OPDVR update direction is
\begin{equation}
\Delta_{\mathrm{OPDVR},t} = R \, \mathrm{ReLU}(R r_t) \, u_t,
\label{eq:opdvr-update}
\end{equation}
with $\mathrm{ReLU}(x)=\max(0,x)$.

Recall that the pure verifier-driven RLVR update direction is $\Delta_{\mathrm{RLVR},t}=R u_t$. The following proposition shows that OPDVR is always aligned with this verifier direction, while OPD may be anti-aligned.

\begin{proposition}[Verifier alignment]
For any token with $u_t \neq 0$, the OPDVR update is never in the opposite direction of the verifier gradient. Specifically,
\begin{equation}
\langle \Delta_{\mathrm{OPDVR},t}, \Delta_{\mathrm{RLVR},t}\rangle \ge 0.
\end{equation}
In contrast, sampled-token OPD satisfies
\begin{equation}
\langle \Delta_{\mathrm{OPD},t}, \Delta_{\mathrm{RLVR},t}\rangle = r_t R \|u_t\|^2,
\end{equation}
which becomes negative whenever $r_t$ and $R$ have opposite signs.
\end{proposition}

\begin{proof}
Using Equation~\ref{eq:opdvr-update} and $\Delta_{\mathrm{RLVR},t}=R u_t$, we have
\begin{align}
\langle \Delta_{\mathrm{OPDVR},t}, \Delta_{\mathrm{RLVR},t}\rangle
&= \langle R \, \mathrm{ReLU}(R r_t) u_t, R u_t \rangle \notag \\
&= R^2 \, \mathrm{ReLU}(R r_t) \, \|u_t\|^2 \notag \\
&= \mathrm{ReLU}(R r_t) \, \|u_t\|^2 \ge 0.
\end{align}
For sampled-token OPD, using Equation~\ref{eq:opd-update} gives
\begin{equation}
\langle \Delta_{\mathrm{OPD},t}, \Delta_{\mathrm{RLVR},t}\rangle
= \langle r_t u_t, R u_t \rangle
= r_t R \|u_t\|^2,
\end{equation}
which is negative when $r_t R < 0$.
\end{proof}

The condition $r_t R < 0$ corresponds exactly to the two conflict cases identified earlier: (i) a correct trajectory with $\pi_\theta > \pi_T$, and (ii) an incorrect trajectory with $\pi_T > \pi_\theta$. In both cases, OPD pushes the policy in a direction that reduces the verifier reward.

\subsection{Decomposition of the OPD Gradient}

The gated mechanism can be interpreted as explicitly removing the verifier-conflicting component from the OPD gradient.

\begin{proposition}[Conflict removal]
The sampled-token OPD update can be decomposed as
\begin{equation}
\Delta_{\mathrm{OPD},t} = \Delta_{\mathrm{OPDVR},t} + \Delta_{\mathrm{conflict},t},
\end{equation}
where
\begin{equation}
\Delta_{\mathrm{conflict},t} = r_t \, \mathbf{1}(r_t R < 0) \, u_t.
\end{equation}
Moreover, the conflict term always has a non-positive projection onto the verifier gradient:
\begin{equation}
\langle \Delta_{\mathrm{conflict},t}, \Delta_{\mathrm{RLVR},t}\rangle
= -|r_t| \, \mathbf{1}(r_t R < 0) \, \|u_t\|^2 \le 0.
\end{equation}
\end{proposition}

\begin{proof}
When $r_t R \ge 0$, we have $R \, \mathrm{ReLU}(R r_t) = r_t$, so $\Delta_{\mathrm{OPDVR},t} = r_t u_t$ and the conflict term vanishes. When $r_t R < 0$, we have $\mathrm{ReLU}(R r_t)=0$, so $\Delta_{\mathrm{OPDVR},t}=0$ and the conflict term equals $\Delta_{\mathrm{OPD},t}$. The projection onto $\Delta_{\mathrm{RLVR},t}$ follows directly.
\end{proof}

Thus OPDVR is not an arbitrary modification of OPD; it is precisely OPD with the harmful verifier-opposing component removed.

\subsection{A Simplified Token-Level Analysis: OPDVR Can Strictly Outperform the Teacher}

To illustrate the mechanism while retaining the token-level structure of LLM generation, we consider a simplified single-token decision problem. Let $s$ be a fixed prefix. At this prefix, the model must choose between two candidate next tokens: $a_1$ denotes a correct key token that leads to a correct final answer, and $a_2$ denotes an incorrect key token that leads to a wrong final answer. The trajectory-level verifier reward is therefore $R(a_1)=+1$ and $R(a_2)=-1$.

Let $\pi_\theta(a_1 \mid s)=q$ and $\pi_T(a_1 \mid s)=p$. Assume the teacher is suboptimal, i.e., $p < \tfrac12$, and that the initial student policy is already better than the teacher, i.e., $q_0 > p$.

Under sampled-token OPD, the expected loss reduces to the reverse KL divergence
\begin{equation}
\mathcal{L}_{\mathrm{OPD}} = D_{\mathrm{KL}}(\pi_\theta(\cdot \mid s) \,\|\, \pi_T(\cdot \mid s)),
\end{equation}
whose global minimizer is $\pi_\theta(\cdot \mid s) = \pi_T(\cdot \mid s)$. Hence the OPD optimum satisfies $q = p$, with expected verifier reward
\begin{equation}
J_{\mathrm{OPD}} = 2p - 1.
\end{equation}

Under OPDVR, the expected update for token $a_1$ is
\begin{equation}
R(a_1) \, \mathrm{ReLU}\!\left(R(a_1) \log \frac{p}{q_0}\right)
=
\mathrm{ReLU}\!\left(\log \frac{p}{q_0}\right)
= 0,
\end{equation}
because $\log(p/q_0)<0$ since $q_0>p$. Similarly, for token $a_2$ we have $R(a_2)=-1$ and
\begin{equation}
\log \frac{1-p}{1-q_0} > 0
\end{equation}
since $q_0>p$, so the OPDVR update for $a_2$ is
\begin{equation}
R(a_2) \, \mathrm{ReLU}\!\left(R(a_2) \log \frac{1-p}{1-q_0}\right)
=
-\mathrm{ReLU}\!\left(-\log \frac{1-p}{1-q_0}\right)
= 0.
\end{equation}
Thus the expected OPDVR gradient vanishes, and the policy remains at $q_0$. Consequently,
\begin{equation}
J_{\mathrm{OPDVR}} = 2q_0 - 1 > 2p - 1 = J_{\mathrm{OPD}} = J_{\mathrm{teacher}}.
\end{equation}
Therefore, OPDVR strictly outperforms both standard sampled-token OPD and the teacher policy.

\section{Implementation Details}
\subsection{Hyperparameters}
\label{app:hyperparameters}

We provide the detailed hyperparameter configurations used in our experiments in Table~\ref{tab:hyperparameters}. All models are trained using the \textbf{Verl}~\cite{verl} framework with the settings specified below.

\begin{table}[H]
    \centering
    \caption{Hyperparameter settings.}
    \label{tab:hyperparameters}
    \vspace{5pt}
    \begin{tabular}{l|c}
        \toprule
        \textbf{Hyperparameter} & \textbf{Value} \\
        \midrule
        Learning Rate & 1e-6 \\
        Train Batch Size & 256 \\
        PPO Mini-Batch Size & 256 \\
        Max Response Length & 8192 \\
        Max Prompt Length & 1024 \\
        Rollout Temperature & 1.0 \\
        Evaluation Temperature & 0.7 \\
        Evaluation Top-$p$ & 0.95 \\
        \bottomrule
    \end{tabular}
\end{table}

\subsection{Baselines Implementation Details}
\label{app:baselines}
\textbf{OPD+GRPO~\citep{xiao2026mimo}.} We follow the original OPD objective and add a GRPO-based policy gradient loss with equal weighting (1:1) between the two terms. Specifically, the total loss is $\mathcal{L} = \mathcal{L}_{\text{OPD}} + \mathcal{L}_{\text{GRPO}}$.

\textbf{RLSD~\citep{yang2026self}.} We set the group size to 8. Since we compare against on-policy distillation methods, we replace the original privileged self-model teacher with the same teacher used by OPD to ensure a fair comparison.

\textbf{ExOPD~\citep{yang2026learning}.} Following the original paper, we set the extrapolation coefficient to 1.25, which was reported as the best setting.

\textbf{SG-OPD~\citep{xu2026sg}.} We set the extrapolation and interpolation coefficients to 1.8 and 1.0, respectively, matching the best configuration reported in the original paper.

\section{Hardware Setup}

All experiments in this paper are conducted on NVIDIA GeForce RTX 5090 GPUs.

\end{document}